%% file: paper.tex
\documentclass{article}

\usepackage{microtype}
\usepackage{graphicx}
\usepackage{subcaption}
\usepackage{booktabs} 

\usepackage{hyperref}

\usepackage[preprint]{icml2026}

\usepackage{amsmath}
\usepackage{amssymb}
\usepackage{mathtools}
\usepackage{amsthm}

\usepackage{bm}
\newcommand{\argmax}{\mathop{\rm argmax}\displaylimits}
\newcommand{\argmin}{\mathop{\rm argmin}\displaylimits}

\usepackage[capitalize,noabbrev]{cleveref}

\theoremstyle{plain}
\newtheorem{theorem}{Theorem}[section]

\newtheorem{lemma}[theorem]{Lemma}
\newtheorem{corollary}[theorem]{Corollary}
\theoremstyle{definition}

\theoremstyle{remark}

\usepackage{takeuchi_mcr}

\usepackage[textsize=tiny]{todonotes}

\icmltitlerunning{Safe Distributionally Robust Feature Selection under Covariate Shift}

\begin{document}

\twocolumn[
  \icmltitle{Safe Distributionally Robust Feature Selection under Covariate Shift}



  \icmlsetsymbol{equal}{*}

  \begin{icmlauthorlist}
    \icmlauthor{Hiroyuki Hanada}{nagoyau}
    \icmlauthor{Satoshi Akahane}{nagoyau}
    \icmlauthor{Noriaki Hashimoto}{riken}
    \icmlauthor{Shion Takeno}{nagoyau}
    \icmlauthor{Ichiro Takeuchi}{nagoyau,riken}
  \end{icmlauthorlist}

  \icmlaffiliation{nagoyau}{Graduate School of Engineering, Nagoya University, Nagoya, Japan}
  \icmlaffiliation{riken}{Center for Advanced Intelligence Project, RIKEN, Tokyo, Japan}

  \icmlcorrespondingauthor{Hiroyuki Hanada}{hanada.hiroyuki.i9@f.mail.nagoya-u.ac.jp}
  \icmlcorrespondingauthor{Ichiro Takeuchi}{takeuchi.ichiro.n6@f.mail.nagoya-u.ac.jp }

  \icmlkeywords{Feature reduction, Safe screening, Covariate shift, Distributionally robust optimization}

  \vskip 0.3in
]



\printAffiliationsAndNotice{}  

\begin{abstract}
\input{abstract}
\end{abstract}

\input{sec1}

\input{sec2}

\input{sec3}

\input{sec4}
\input{sec5}

%
%

\section*{Impact Statement}

This paper aims to advance existing machine learning methods and their potential applications.

We do not foresee any specific societal impacts that require separate discussion beyond general considerations for responsible use.
%
%
%

%
\bibliography{paper}
\bibliographystyle{icml2026}

\newpage
\appendix
\onecolumn

\input{app1}
\input{app2}

\end{document}

%% file: abstract.tex
In practical machine learning, the environments encountered during the model development and deployment phases often differ, especially when a model is used by many users in diverse settings.
Learning models that maintain reliable performance across plausible deployment environments is known as \emph{distributionally robust (DR) learning}.
In this work, we study the problem of \emph{distributionally robust feature selection (DRFS)}, with a particular focus on sparse sensing applications motivated by industrial needs.
In practical multi-sensor systems, a shared subset of sensors is typically selected prior to deployment based on performance evaluations using many available sensors.
At deployment, individual users may further adapt or fine-tune models to their specific environments.
When deployment environments differ from those anticipated during development, this strategy can result in systems lacking sensors required for optimal performance.
To address this issue, we propose \emph{safe-DRFS}, a novel approach that extends safe screening from conventional sparse modeling settings to a DR setting under covariate shift.
Our method identifies a feature subset that encompasses all subsets that may become optimal across a specified range of input distribution shifts, with finite-sample theoretical guarantees of no false feature elimination.

%% file: sec1.tex
\section{Introduction}
\label{sec:intro}
In practical machine learning, model development and deployment are distinct phases, and uncertainty often arises in the deployment phase due to variations in operating environments.
\emph{Distributionally robust (DR) learning}~\cite{mohajerin2018data} has therefore attracted significant attention as a framework for handling such uncertainty.
In its standard formulation, DR learning assumes that a model is fixed in the development phase and deployed without modification, with the goal of guaranteeing performance under the worst-case distribution within a prescribed uncertainty set.
More recently, advanced DR learning settings have been explored in which the development phase determines a model backbone—such as a representation or structure—while allowing users to fine-tune or retrain models in their deployment environments.
In this case, the goal is to construct a robust foundation that ensures satisfactory performance even under worst-case deployment scenarios accounting for such adaptation.

In this work, we study such an advanced DR learning setting that we term \emph{DR feature selection (DRFS)}.
In the development phase, a system developer selects a subset of sensors (features) from a large pool of candidates, which are then physically equipped in a multi-sensor system and fixed at release.
In the deployment phase, the released system is used by many individual users operating in different environments, each of whom fine-tunes or retrains a model by selecting useful sensors from the fixed subset.
When deployment environments are uncertain and vary across users, the sensor subset must therefore be constructed in a DR manner.
This problem is motivated by industrial needs, where developers must equip systems with all sensors that may be required under uncertain usage conditions while eliminating unnecessary ones to reduce hardware cost and system complexity.

\begin{figure*}[t]
 \centering
 \igr{1.00}{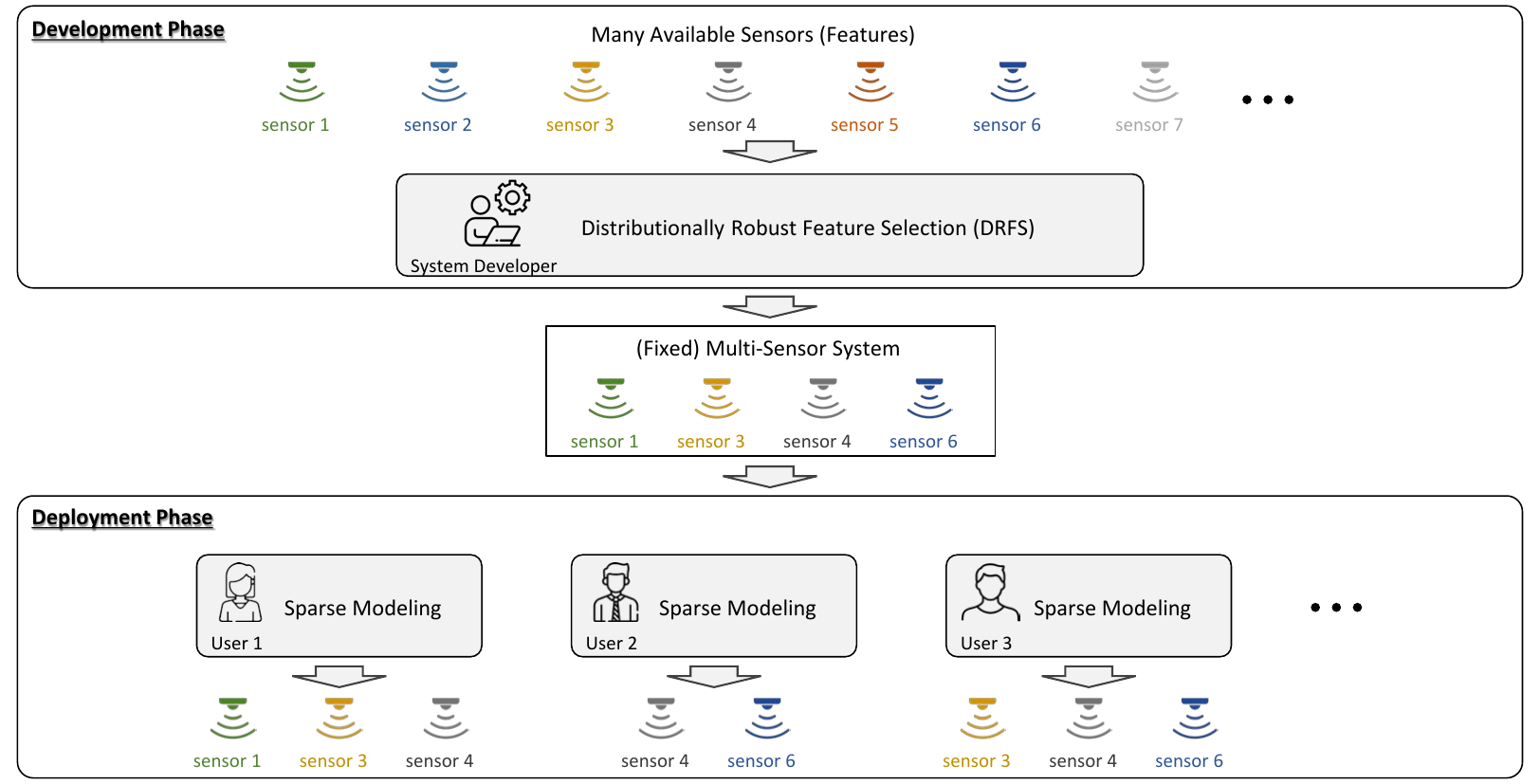}
 \caption{
This figure illustrates the sparse sensing problem studied in this work, which consists of two phases: a development phase (top) and a deployment phase (bottom).
In the development phase, a system developer selects a subset of sensors from a large pool of candidates and designs a system equipped with these sensors, after which the sensor set is fixed.
In the deployment phase, the system is used by many end users operating in diverse environments with uncertainty, modeled as covariate shift, where the input distribution varies across users within a specified range.
Each user adapts or fine-tunes a sparse model, such as a regression or classification model, using only the sensors available in the system.
Under covariate shift, the optimal support of sparse models may differ across users, leading to different sensor requirements.
The key challenge for the developer is therefore to select, during development, all sensors that may be required under plausible deployment environments while safely eliminating sensors that will never be used.
  }
  \label{fig:schematic_illustration}
\end{figure*}

As a specific instance of the DRFS problem, we consider sparse feature selection for regression and classification models under deployment environments subject to \emph{covariate shift}~\cite{shimodaira2000improving}.
Specifically, users operate in environments where the input distribution deviates from that observed during development.
In each deployment environment, users construct sparse models by selecting features from the fixed sensor set equipped in the system.
The key challenge for the system developer is therefore to determine, during the development phase, a sensor subset that can cover all features that may be selected across a range of plausible deployment environments.
Figure~1 illustrates the problem setting considered in this work.

Since the system developer does not know in advance the environments in which the system will be deployed, sensors must be selected by accounting for worst-case covariate shifts within a prescribed range.
Our key idea is to bring concepts originally developed in the context of \emph{safe screening}~\cite{elghaoui2012safe} into this DR learning setting.
Safe screening was originally developed to accelerate sparse optimization by identifying surely unnecessary features in the sense that their prediction coefficients are zero (called \emph{inactive}) and eliminating them without affecting optimality.
We extend these screening criteria to certify feature relevance across environment-dependent covariate shifts.
This extension enables the identification of sensors required under any plausible deployment environment and provides a complete finite-sample theoretical guarantee that no necessary features are ever discarded; we thus refer to the proposed approach as \emph{safe-DRFS}.

\subsection*{Related Works}

\paragraph{Distributionally Robust (DR) learning}
In the classical DR learning setup, a model is trained once in the development phase and deployed without further modification.
The goal is to optimize the worst-case expected loss over an uncertainty set of test distributions, guaranteeing reliable performance under the most adverse shift within the prescribed set~\cite{namkoong2016stochastic,shafieezadeh2015distributionally,mohajerin2018data,namkoong2017variance,sinha2018certifying}.
While these formulations typically consider general forms of distributional uncertainty, including shifts in both input and output distributions, a prominent line of work focuses specifically on uncertainty in the input distribution, commonly modeled as covariate shift~\cite{shimodaira2000improving,sugiyama2007covariate,kouw2019review}.

More recently, an advanced DR learning setting has emerged that assumes a two-stage workflow, where a developer learns a robust base during development and users adapt it at deployment via fine-tuning or test-time updates.
This perspective connects DR learning to test-time adaptation, personalization, and meta-learning, and has been explored in recent work allowing deployment-time adaptation under distribution shift~\cite{wang2021tent,sun2020test,zhou2021domain}.

A closely related recent work is \cite{swaroop2025distributionally}, which studies feature selection under distribution shift.
They propose a DR objective function and tractable solution methods based on continuous relaxations and heuristic optimization, but do not provide exact guarantees for the original robust problem.
In contrast, our approach focuses on \emph{safe certification} by deriving explicit, computable worst-case bounds that enable feature elimination with finite-sample guarantees.

\paragraph{Safe screening}
Safe screening techniques are designed to accelerate sparse learning by identifying features that are guaranteed to be inactive in the optimal solution and can therefore be safely removed without affecting optimality.
The first such rule was proposed by \citet{elghaoui2012safe}, followed by numerous extensions based on different optimality bounds~\cite{wang2013lasso,liu2014safe,wang2014safe,xiang2011screening}.
Among them, duality gap (DG)-based screening rules are particularly attractive due to their generality and ease of computation~\cite{fercoq2015mind,ndiaye2015gap,ndiaye2017gap}.

Beyond standard sparse feature selection, the safe screening principles have been extended and applied to a variety of problems, including sample screening~\cite{ogawa2013safe,shibagaki2016simultaneous}, pattern and rule mining~\cite{nakagawa2016safe,kato2023saferulefit}, model selection~\cite{okumura2015quick,hanada2018efficiently}, and finite-sample error bound computation~\cite{shibagaki2015regularization,ndiaye2019approximatehomotopy}.
However, to the best of our knowledge, safe screening principles have not previously been applied to DR learning.

\subsection*{Main Contributions}
This paper makes three main contributions.
First, we formulate, for the first time, a DR feature selection problem that explicitly accounts for deployment-time adaptation based on the safeness, that is, eliminating features that are assured not to affect the prediction model in the deployment phase.
Second, to achieve this, we newly extended the methodology of safe screening --- originally developed for accelerating sparse optimization --- to be applied to DR learning, and demonstrate their effectiveness for robust feature selection.
Finally, we derive computable worst-case bounds that enable DR feature selection with finite-sample theoretical guarantees of no false feature elimination, and validate our theoretical results through empirical evaluations.

%% file: sec2.tex
\section{Problem Setup}
\label{sec:problem_setup}
We consider a two-phase setting in multi-sensor systems.
In the development phase, a system developer selects a subset of sensors to be physically equipped in the system, which defines the features available for subsequent modeling.
In the deployment phase, end users operate the fixed system in their own environments and adapt sparse models using the available sensors.
Since deployment environments may vary across users, the system must include all sensors that could be required under plausible input distribution shifts.
Our goal is therefore to select, during development, a feature subset that remains sufficient across a specified range of deployment environments.

\paragraph{Development phase}
During the development phase, we are given a dataset denoted by $\{(\bm x_i, y_i)\}_{i \in [n]}$, where $\bm x_i \in \cX \subseteq \RR^d$ is a $d$-dimensional feature vector defined in a domain $\cX$ and $y_i \in \cY$ is the corresponding response variable.
We consider both regression and binary classification problems.
In the regression setting, the response space satisfies $\cY \subseteq \RR$, whereas in the binary classification setting, $\cY = \{-1, 1\}$.
Here, the notation $[n] = \{1, 2, \ldots, n\}$ denotes the set of natural numbers from $1$ to $n$.

We consider a linear model of the form
\begin{align*}
 f(\bm x_i; \bm b, b_0) = \bm x_i^\top \bm b + b_0, ~ i \in [n], 
\end{align*}
where $\bm b \in \RR^d$ is a vector of feature coefficients and $b_0 \in \RR$ is the intercept term.
Throughout the paper, we use the notation $\bm \beta := (\bm b, b_0)$ to express the feature coefficient vector and the intercept term collectively\footnote{
In what follows, we use the notation $\bm \beta$ to collectively represent a feature coefficient vector and an intercept term; for example, expressions such as $\bm \beta^*$ or $\hat{\bm \beta}$ correspond to the pairs $(\bm b^*, b_0^*)$ and $(\hat{\bm b}, \hat b_0)$, respectively.
}.

To induce sparsity in the model, we employ $L_1$-regularized estimation.
Specifically, for both regression and binary classification problems, we consider empirical risk minimization (ERM) with an $L_1$ regularization term, given by
\begin{align}
 \nonumber
 \bm \beta^*
 :=&
 \left(\bm b^*, b_0^*\right)
 \\
 \label{eq:erm}
 =&
 \argmin_{\bm b \in \RR^d,\, b_0 \in \RR}
 \sum_{i \in [n]}
 \ell\bigl(y_i, \bm x_i^\top \bm b + b_0 \bigr)
 + \lambda \|\bm b\|_1,
\end{align}
where $\ell(\cdot,\cdot)$ denotes a loss function and $\lambda > 0$ is a regularization parameter controlling the sparsity of the solution.
In the regression setting, $\ell$ corresponds to a regression loss such as the squared loss $\ell(y_i, t) := (t - y_i)^2$, while in the binary classification setting, $\ell$ represents a classification loss such as the logistic loss $\ell(y_i, t) := \log(1 + e^{-yt})$.
Due to the effect of sparse regularization, the optimal coefficient vector $\bm b^*$ is sparse.
We denote by $\cS(\bm \beta^*) \subseteq [d]$ the set of indices corresponding to the nonzero feature coefficients.

\paragraph{Deployment phase}
In the deployment phase, we consider a setting in which the input distribution induced by each user's operating environment differs from that of the development phase within a certain range.
Learning under such distributional discrepancies between development and deployment environments is commonly referred to as \emph{covariate shift}.
Under covariate shift, it is typically assumed that the conditional distribution of the response given the input remains unchanged, while the marginal input distribution varies across environments.
A standard approach to addressing covariate shift is to formulate the learning problem as a weighted ERM problem.
Specifically, let $P_{\rm dev}(\bm x)$ and $P_{\rm dep}(\bm x)$ denote the input distributions in the development and deployment environments, respectively.
For each training instance $(\bm x_i, y_i)$, we define a weight
\begin{align}
 \label{eq:weight}
 w_i = \frac{P_{\rm dep}(\bm x_i)}{P_{\rm dev}(\bm x_i)}, ~ i \in [n].
\end{align}
Assuming that the same dataset $\{(\bm x_i, y_i)\}_{i \in [n]}$ is available in the deployment phase, the weighted ERM problem is formulated using these weights as
\begin{align}
 \label{eq:weighted_erm}
 \bm \beta^{*(\bm w)} :=& \left(\bm b^{*(\bm w)}, b_0^{*(\bm w)}\right) \\
 \nonumber
 =& \argmin_{\bm b \in \RR^d, b_0 \in \RR}
 \underbrace{\sum_{i \in [n]} w_i \ell\bigl(y_i, \bm x_i^\top \bm b + b_0 \bigr) + \lambda \|\bm b\|_1}_{=: \cP^{(\bm w)}(\bm \beta)}, 
\end{align}
where $\bm w = (w_1, \ldots, w_n)^\top$.
Under the assumption of the covariate shift, the model trained with these weights is known to achieve consistency \cite{shimodaira2000improving,sugiyama2007covariate}.
The sparse solution obtained by the weighted ERM in \eq{eq:weighted_erm} has a set of nonzero coefficients $\cS(\bm \beta^{*(\bm w)})$ that does not necessarily coincide with $\cS(\bm \beta^*)$.

\paragraph{DR feature subset}
We consider a setting in which the input distributions at the deployment phase differ across users.
Namely, the deployment input distribution $P_{\rm dep}(\bm x)$ is allowed to vary within a specified range.
Under the covariate shift formulation, this corresponds to considering a range of possible importance weights defined in \eq{eq:weight}.
Let $\cW_\delta$ denote the set of all plausible weight vectors $\bm w$ induced by the considered range of deployment input distributions.

Noting that $w_i = 1$ for all $i \in [n]$ in the development phase, we specify the admissible set of weight vectors as
\begin{align}
 \label{eq:weight_range}
 \cW_\delta
 :=
 \left\{
 \bm w \in [1-\delta, 1+\delta]^n
 \left|
 \sum_{i \in [n]} w_i = n
 \right.
 \right\},
\end{align}
where $\delta \in (0,1)$ is a constant.
This allows each weight to vary within a range of $\pm\delta$ around its development-phase value while keeping the total weight $\sum_{i \in [n]} w_i$ fixed at $n$.

For each $\bm w \in \cW_\delta$, we can solve the corresponding weighted ERM problem and obtain an optimal sparse solution $\bm \beta^{*(\bm w)}$, whose support $\cS(\bm \beta^{*(\bm w)})$ represents the set of features selected under that particular deployment environment.
This means that the system must be equipped with all features that may become necessary for any plausible deployment environment.
This leads to selecting the feature set
\begin{align}
 \label{eq:set_star}
 \cS_{\rm DR} = \bigcup_{\bm w \in \cW_\delta} \cS(\bm \beta^{*(\bm w)}), 
\end{align}
which consists of all features that are selected by at least one optimal sparse model corresponding to a feasible deployment environment.
Equipping the system with $\cS_{\rm DR}$ ensures that, for any user whose operating environment falls within the specified range, the deployment-phase model can achieve optimal performance without being limited by missing features.

If the feature set $\cS_{\rm DR}$ could be identified exactly, a system developer can equip the system with these features (sensors).
However, $\cS_{\rm DR}$ is defined through the solutions of \eq{eq:weighted_erm} over infinitely many weight vectors $\bm w \in \cW_\delta$, making it impractical to compute directly.
In this work, rather than attempting to recover $\cS_{\rm DR}$ directly, we propose a method for identifying a superset $\hat{\cS}_{\rm DR} \supseteq \cS_{\rm DR}$.
By equipping the system with the features in $\hat{\cS}_{\rm DR}$ during the development phase, a system developer can guarantee that an optimal sparse model can be obtained for any deployment environment within the specified range.

%% file: sec3.tex
\section{Proposed Method: Safe Distributionally Robust Feature Selection (Safe-DRFS)}
\label{sec:proposed_method}
In this section, we describe our proposed method, called \emph{safe-DRFS}.
As discussed in Section~\ref{sec:problem_setup}, our goal is to enable a system developer to identify features (sensors) that may be required under uncertain deployment environments.
We will achieve the goal by identifying features guaranteed to be inactive (coefficients zero) across all plausible deployment environments.
Such features can be safely eliminated during development without affecting deployment-time optimality.
The proposed \emph{safe-DRFS} method is based on this perspective and identifies features that are guaranteed not to be selected by the weighted sparse learning problem in \eq{eq:weighted_erm} for any deployment environment.
By eliminating these features in advance, \emph{safe-DRFS} enables robust and cost-effective system design under uncertainty.

The proposed \emph{safe-DRFS} method is grounded in convex optimization theory.
When the loss function $\ell(\cdot,\cdot)$ in \eq{eq:erm} and \eq{eq:weighted_erm} is a convex function satisfying certain regularity conditions, the optimality conditions of sparse solutions explicitly characterize when individual feature coefficients are zero or nonzero; we refer to these as \emph{sparseness conditions}.
These conditions can be expressed as inequalities involving the optimal solution of the corresponding dual problem and can be evaluated once optimality is achieved.
The key idea of \emph{safe-DRFS} is to derive bounds on the sparseness conditions over a range of deployment-phase weight vectors, thereby identifying features whose coefficients are guaranteed to be zero for all $\bm w \in \cW_\delta$.

To derive bounds on the sparseness conditions, we build upon established results from the safe screening literature (see Section~\ref{sec:intro}).
Safe screening refers to a class of techniques for accelerating sparse optimization by identifying features whose coefficients are guaranteed to be zero at the optimum, either before optimization or during its execution.
The key principle underlying safe screening is the use of bounds on the sparseness conditions to eliminate inactive features without affecting optimality.
Among the various bounds developed for safe screening, \emph{safe-DRFS} takes the \emph{duality gap (DG) bound}~\cite{fercoq2015mind,ndiaye2015gap,ndiaye2017gap} as its base and extends it to accommodate uncertainty in deployment environments.

\subsection{Sparseness Conditions}
Here, we summarize the sparseness conditions associated with the weighted ERM problem in \eq{eq:weighted_erm}.
We refer to \eq{eq:weighted_erm} as the \emph{primal problem} and call its optimal solution $\bm \beta^{*(\bm w)}$ the \emph{primal solution}.

We assume that the loss function $\ell(\cdot,\cdot)$ in \eq{eq:weighted_erm} is closed and convex in its second argument, and is twice continuously differentiable with a bounded second derivative.
Under this assumption, the corresponding dual problem can be derived via Fenchel duality \cite{rockafellar1970convex} and is given by
\begin{align}
 \nonumber
 \bm \alpha^{*(\bm w)} = & \argmax_{\bm \alpha \in \RR^n} \underbrace{-\sum_{i \in [n]} w_i \ell^*(y_i, -\alpha_i)}_{=: \cD^{(\bm w)}(\bm \alpha)}
 \\
 \label{eq:dual}
 & \text{s.t.} ~ \left\| \sum_{i \in [n]} w_i \alpha_i \bm x_i \right\|_\infty \le \lambda,
 \sum_{i \in [n]} w_i \alpha_i = 0,
\end{align}
where $\ell^*(y, \cdot)$ denotes the Fenchel conjugate of $\ell(y,\cdot)$ with respect to its second argument: $\ell^*(y, t) := \sup_{s}(st - \ell(y, s))$.

Then, the optimality conditions of the primal--dual pair imply the following \emph{sparseness condition}: for each feature index $j \in [d]$,
\begin{align}
\label{eq:sparseness_condition}
\left|
\sum_{i \in [n]} w_i \alpha_i^{*(\bm w)} x_{ij}
\right|
<
\lambda
~\Rightarrow~
b^{*(\bm w)}_j = 0.
\end{align}
Namely, if the absolute inner product between the $j$-th feature and the optimal dual variable is strictly smaller than the regularization parameter $\lambda$, then the corresponding coefficient in the primal optimal solution must be zero.

\subsection{Safe Screening} \label{sec:safescreening}
The sparseness condition (left-hand side of \eq{eq:sparseness_condition}) is expressed in terms of the optimal dual solution $\bm \alpha^{*(\bm w)}$ and therefore cannot be directly evaluated before solving the optimization problem.
The core idea of safe screening is to determine feature sparsity by computing an upper bound on the left-hand side of \eq{eq:sparseness_condition} prior to obtaining the optimal solution $\bm \alpha^{*(\bm w)}$, or more specifically, by identifying a set that $\bm \alpha^{*(\bm w)}$ surely exists. Then we check whether this upper bound is smaller than $\lambda$.
In this work, to identify the set that $\bm \alpha^{*(\bm w)}$ surely exist,
we employ DG bound presented later as \eqref{eq:duality-gap-ball}.
Below, we formulate the DG bound in the context of our problem.

%
Let $\cP^{(\bm w)} : \RR^{d+1} \to \RR$ and $\cD^{(\bm w)} : \RR^n \to \RR$ be the ones defined in \eqref{eq:weighted_erm} and \eqref{eq:dual}, respectively.
For any primal feasible solution $\hat{\bm \beta} \in \RR^{d+1}$ and any dual feasible solution $\hat{\bm \alpha} \in \RR^n$, the DG is defined as
\begin{align*}
G^{(\bm w)}(\hat{\bm \beta}, \hat{\bm \alpha}) := \cP^{(\bm w)}(\hat{\bm \beta}) - \cD^{(\bm w)}(\hat{\bm \alpha}).
\end{align*}

The DG provides a bound on the distance between the dual feasible solution and the dual optimal solution.
In particular, under the assumptions stated above, using $w_i \ge 1-\delta$ for $\bm w \in \cW_\delta$, the dual optimal solution $\bm \alpha^{*(\bm w)}$ is guaranteed to lie in the Euclidean ball
\begin{align}
\bm \alpha^{*(\bm w)} \in \left\{ \bm \alpha \in \RR^n \middle| \|\bm \alpha - \hat{\bm \alpha}\|_2 \le \sqrt{\frac{2 \nu}{1 - \delta}  G^{(\bm w)}(\hat{\bm \beta}, \hat{\bm \alpha})} \right\},
\label{eq:duality-gap-ball}
\end{align}
where $\nu > 0$ is the maximum Lipschitz constant of the function $g(t) := \frac{d}{dt}\ell(y, t)$ for any $y\in\cY$. For example, $\nu = 2$ for the squared loss while $\nu = 1/4$ for the logistic loss\footnote{Note that, in case of the squared loss and the logistic loss, the choice of $y$ does not affect the Lipschitz constant of $g(t)$.}. See Appendix \ref{app:duality-gap-ball} for the proof.

Using this property, the left-hand side of the sparseness condition in \eq{eq:sparseness_condition} can be upper bounded as
\begin{align}
 \label{eq:upper_bound}
 \left|
\sum_{i \in [n]} w_i \alpha_i^{*(\bm w)} x_{ij}
\right|
\le
\mathrm{UB}_j^{(\bm w)}(\hat{\bm \alpha}, \hat{\bm \beta}),
\end{align}
where the upper bound ${\rm UB}_j^{(\bm w)}$ is presented in the following theorem:
\begin{theorem}
 \label{theo:upper_bound}
 We assume that the loss function $\ell(\cdot,\cdot)$ in \eq{eq:weighted_erm} is closed and convex in its second argument, and is twice continuously differentiable with respect to that argument, with a Lipschitz continuous derivative with constant $\nu$.
 Additionally, we assume that the weights satisfy $w_i > 0$ for all $i \in [n]$.
 Then, the upper bound in \eq{eq:upper_bound} is represented as 
 \begin{align*}
  \mathrm{UB}_j^{(\bm w)}(\hat{\bm \alpha}, \hat{\bm \beta})
  &= 
  \left|
  \sum_{i \in [n]} w_i \hat{\alpha}_i x_{ij}
  \right|
  \\
  &+
  \sqrt{
  \left(
  \sum_{i \in [n]} w_i^2 x_{ij}^2
  \right)
  \left(
  \frac{2 \nu}{1-\delta} G^{(\bm w)}(\hat{\bm \beta}, \hat{\bm \alpha})
  \right)
  }.
 \end{align*}
\end{theorem}
The proof of Theorem~\ref{theo:upper_bound} is presented in Appendix \ref{app:upper_bound}.
If $\mathrm{UB}_j^{(\bm w)}(\hat{\bm \alpha}, \hat{\bm \beta}) < \lambda$, then the corresponding feature is guaranteed to be inactive in the optimal solution for the given $\bm w$.

As the concrete choice of primal-dual feasible solutions $(\hat{\bm \beta}, \hat{\bm \alpha})$, we use 
\begin{align}
 \hat{\bm \beta} = \bm \beta^*,
 \text{ and }
 \hat{\bm \alpha} = q \bm w^{-1} \circ \bm \alpha^*,
 \label{eq:beta-alpha-hat-feasible}
\end{align}
where $\circ$ denotes the Hadamard (element-wise) product, and $q \in (0, 1]$ is loss-function-dependent scalar parameter uniquely determined so that $\hat{\bm \alpha}$ is dual feasible. For example, $q = 1$ for the squared loss while $q = 1 - \delta$ for the logistic loss. See Appendix~\ref{app:feasibility-alpha-hat} for details.

\subsection{Distributionally Robust Sparseness Conditions}

Our goal is to identify features that are guaranteed to never be selected under any plausible deployment environment.
To achieve this, for each feature index $j$, we consider the worst-case value of the upper bound over all admissible weight vectors $\bm w \in \cW_\delta$, defined as
\begin{align}
 \label{eq:upper_bound_all}
 \mathrm{UB}_j := \max_{\bm w \in \cW_\delta} \mathrm{UB}_j^{(\bm w)}(\hat{\bm \beta}, \hat{\bm \alpha}).
\end{align}
If this worst-case upper bound satisfies $\mathrm{UB}_j < \lambda$, then the corresponding feature is guaranteed to be inactive for all deployment environments within the specified range.
Unfortunately, directly computing $\mathrm{UB}_j$ is intractable in general.
Instead, we compute a further tractable upper bound of $\mathrm{UB}_j$ and use it as a criterion for a distributionally robust version of safe feature elimination during the development phase.

\begin{theorem}
 \label{theo:main}
Let
\begin{align*}
 \theta^{\langle j\rangle}_i &:= x_{ij}^2, \\
 \rho_i &:= \ell\bigl(y_i, \bm x_i^\top \bm b^* + b_0^*\bigr)
 \\
 &+ \max\left\{
 \ell^*\!\left(y_i, -\frac{q}{1+\delta}\alpha_i^* \right),
 \ell^*\!\left(y_i, -\frac{q}{1-\delta}\alpha_i^* \right)
 \right\}
\end{align*}
for $i \in [n]$.
Let $\theta^{\langle j\rangle}_{(1)} \le \cdots \le \theta^{\langle j\rangle}_{(n)}$ denote the elements of
$\{\theta^{\langle j\rangle}_i\}_{i \in [n]}$ sorted in nondecreasing order, and similarly,
let $\rho_{(1)} \le \cdots \le \rho_{(n)}$ denote the elements of
$\{\rho_i\}_{i \in [n]}$ sorted in nondecreasing order.
For $i \in [n]$, define
\begin{align}
 w_{(i)}^\#
 &=
 \mycase{
 1-\delta & (i \le \lfloor n/2 \rfloor), \\
 1        & (\lfloor n/2 \rfloor < i \le \lceil n/2 \rceil), \\
 1+\delta & (i > \lceil n/2 \rceil),
 }
 \label{eq:w-worst}
\end{align}
Then, defining
\begin{align*}
 &\hat{\mathrm{UB}}_j
 :=
 q
 \left|
 \sum_{i \in [n]} \alpha_i^* x_{ij}
 \right|
 \\
 &+
 \sqrt{
 \left(
 \sum_{i \in [n]} (w_i^\#)^2 \theta^{\langle j\rangle}_{(i)}
 \right)
 \left(
 \frac{2 \nu}{1 - \delta}
 \left(
 \sum_{i \in [n]} w_i^\# \rho_{(i)}
 + \lambda \| \bm b^* \|_1
 \right)
 \right)
 },
\end{align*}
we obtain a computable upper bound satisfying
$\mathrm{UB}_j \le \hat{\mathrm{UB}}_j$,
which provides a valid upper bound for $\mathrm{UB}_j$ in \eq{eq:upper_bound_all}.
This means that
 \begin{align*}
  \hat{\mathrm{UB}}_j < \lambda
  ~\Rightarrow~
  b_j^{*(\bm w)} = 0
  ~~~
  \forall \bm w \in \cW_\delta.
 \end{align*}
\end{theorem}
The proof of Theorem~\ref{theo:main} is presented in Appendix \ref{app:proof-main}.

If $\hat{\mathrm{UB}}_j < \lambda$, then the $j$-th feature is guaranteed to be inactive for any $\bm w \in \cW_\delta$ during the deployment phase.
Therefore, such a feature (sensor) can be safely removed when releasing the system in the development phase without sacrificing optimality under the considered deployment environments.
The upper bound $\hat{\mathrm{UB}}_j$ given in Theorem~\ref{theo:main} can be computed in $\mathcal{O}(n \log n)$ time.
Consequently, evaluating this bound for all features $j \in [d]$ requires $\mathcal{O}(d n \log n)$ time in total.

%% file: sec4.tex
\section{Numerical Experiments}

\begin{table}[t]
\caption{Datasets for the experiments. ``LIBSVM'': LIBSVM Data \cite{libsvmDataset}, ``UCI'': UCI Machine Learning Repository \cite{Dua2017UCI}. See Appendix \ref{app:experiment-datasets} for the preprocesses conducted on them. (Values of $n$ and $d$ are the ones after the preprocesses.)}
\begin{center}
\label{tab:datasets}
\begin{tabular}{cccc}
\hline
Dataset & $n$ & $d$ & Source\\
\hline
\multicolumn{4}{c}{For regressions} \\
\hline
cpusmall & 8192 & 12 & LIBSVM \\
housing & 506 & 13 & LIBSVM \\
facebook\_comment & 40949 & 52 & UCI \\
triazines & 186 & 58 & LIBSVM \\
online\_news & 39644 & 58 & UCI \\
communities\_and\_crime & 1994 & 101 & UCI \\
superconductivity & 21263 & 158 & UCI \\
blog\_feedback & 52397 & 276 & UCI \\
\hline
\multicolumn{4}{c}{For binary classifications} \\
\hline
australian & 690 & 14 & LIBSVM \\
ionosphere & 351 & 33 & LIBSVM \\
sonar & 208 & 60 & LIBSVM \\
splice & 1000 & 60 & LIBSVM \\
mushrooms & 8124 & 111 & LIBSVM \\
a1a & 1605 & 113 & LIBSVM \\
madelon & 2000 & 500 & LIBSVM \\
colon\_cancer & 62 & 2000 & LIBSVM \\
\hline
\end{tabular}
\end{center}
\end{table}

\begin{figure*}[t]
\includegraphics[width=0.32\hsize]{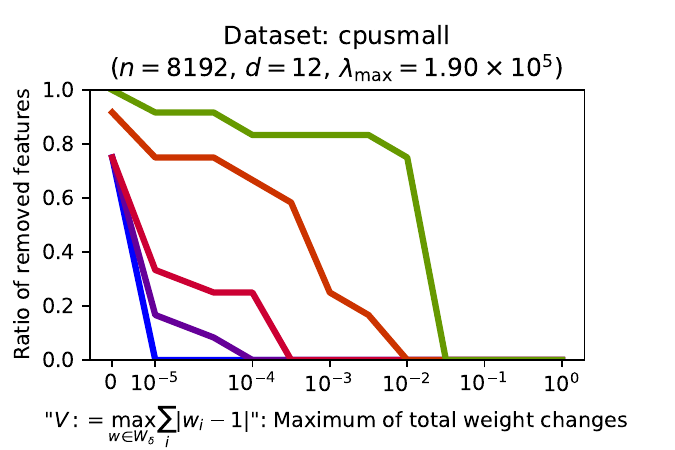}
\includegraphics[width=0.32\hsize]{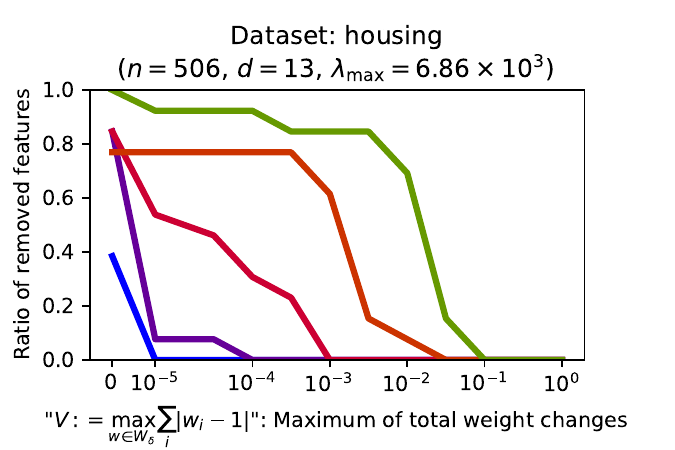}
\includegraphics[width=0.32\hsize]{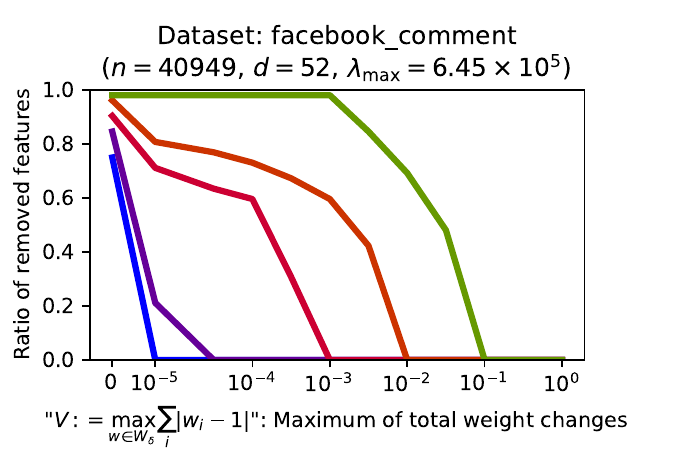}
\\
\includegraphics[width=0.32\hsize]{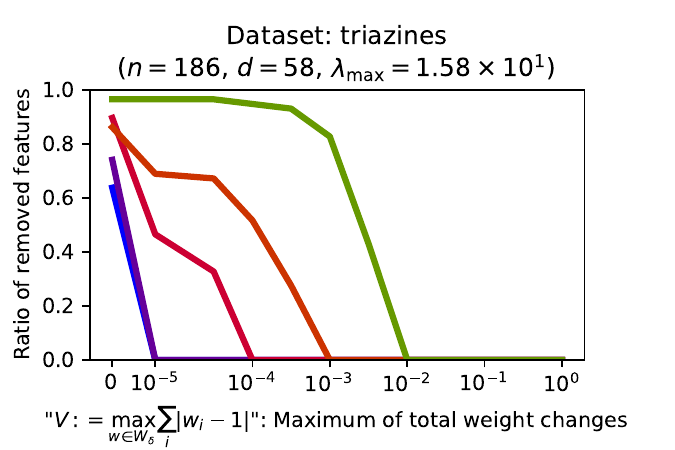}
\includegraphics[width=0.32\hsize]{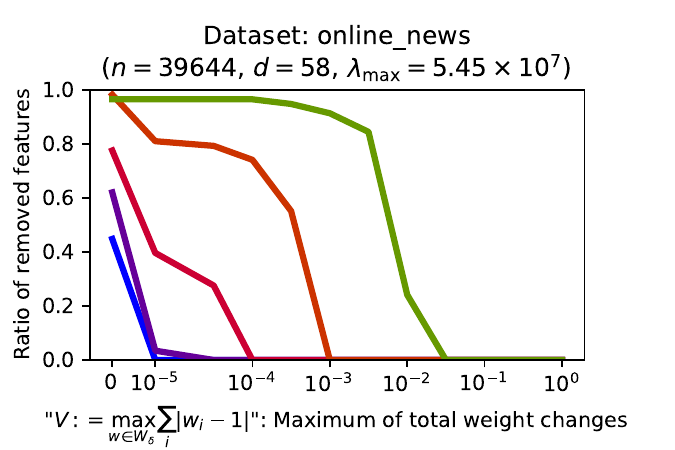}
\includegraphics[width=0.32\hsize]{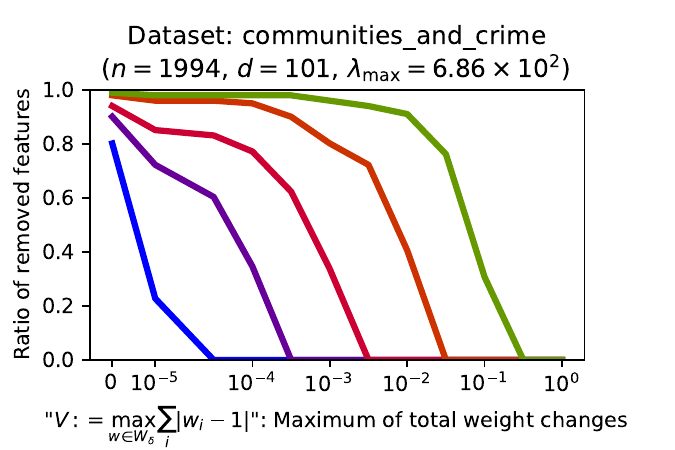}
\\
\includegraphics[width=0.32\hsize]{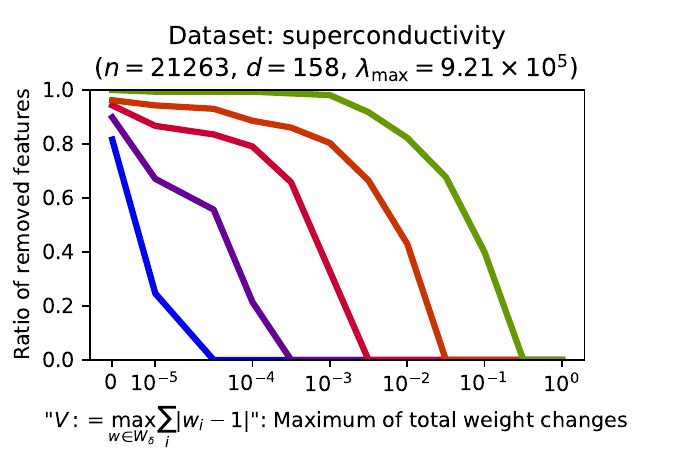}
\includegraphics[width=0.32\hsize]{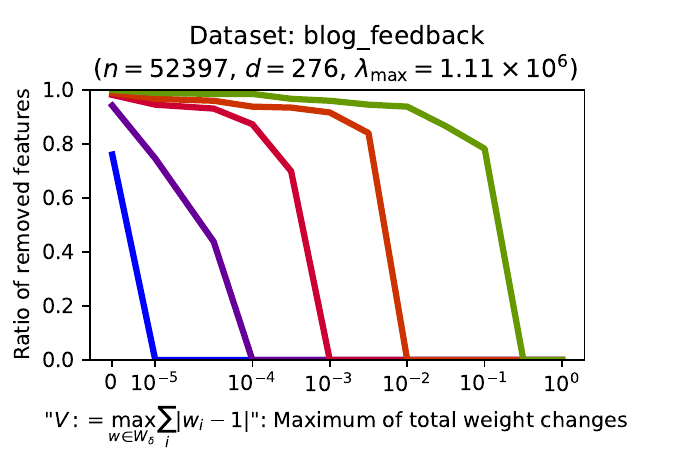}
\includegraphics[width=0.15\hsize]{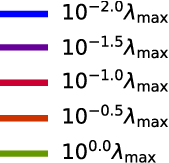}
\caption{Ratios of removed features to the original features for four regression datasets. Horizontal axis: The uncertainty levels of distributional change $V$, Vertical axis: Ratio of reduced features.}
\label{fig:exp-result-regression}
\end{figure*}

\begin{figure*}[t]
\includegraphics[width=0.32\hsize]{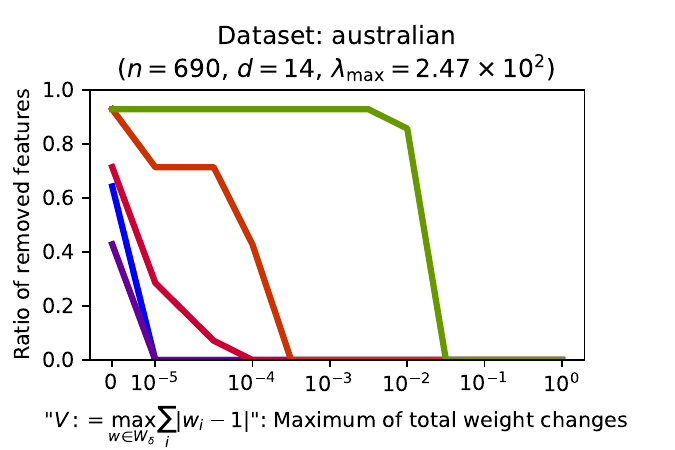}
\includegraphics[width=0.32\hsize]{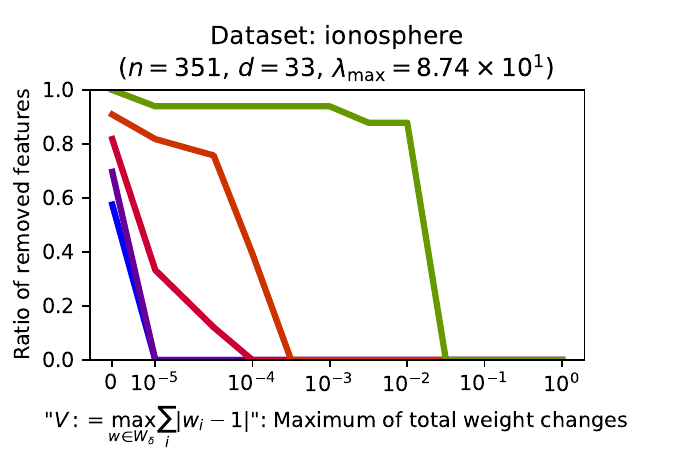}
\includegraphics[width=0.32\hsize]{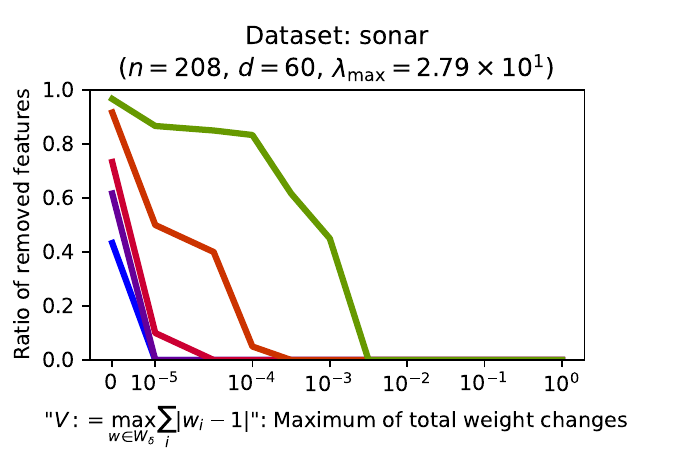}
\\
\includegraphics[width=0.32\hsize]{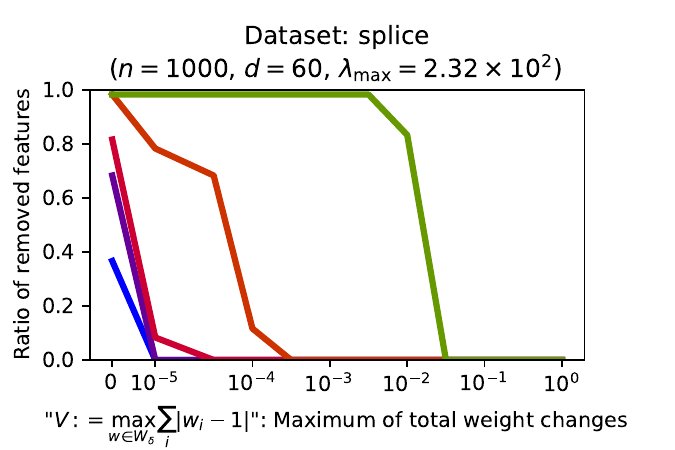}
\includegraphics[width=0.32\hsize]{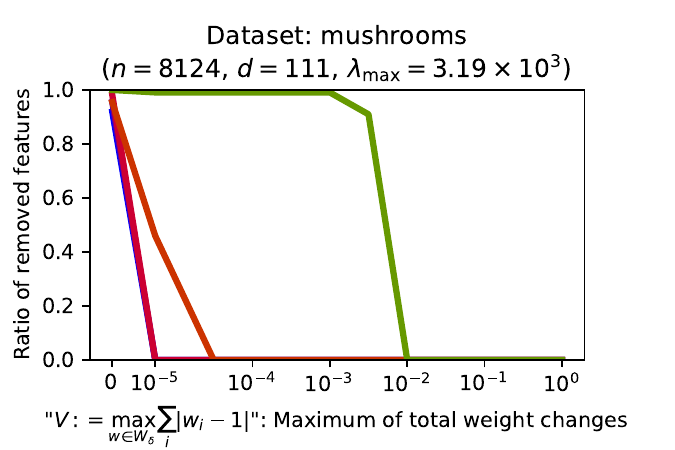}
\includegraphics[width=0.32\hsize]{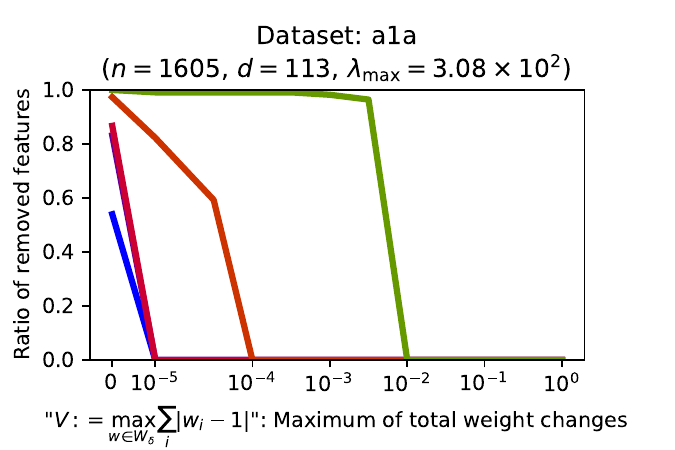}
\\
\includegraphics[width=0.32\hsize]{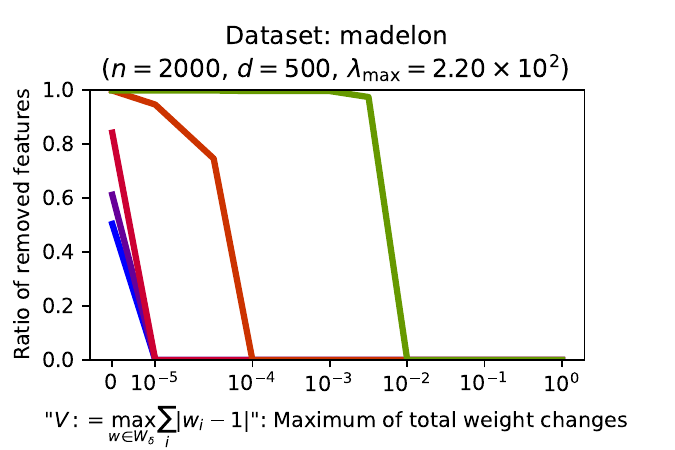}
\includegraphics[width=0.32\hsize]{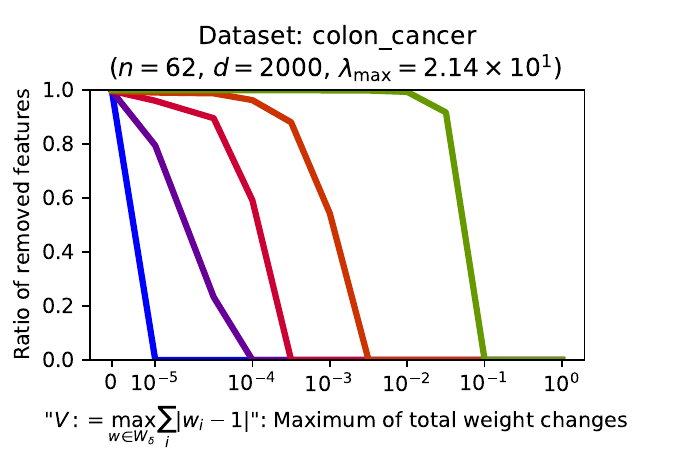}
\includegraphics[width=0.15\hsize]{exp_common_legend.eps}
\caption{Ratio of removed features for eight binary classification datasets to the original features. Horizontal axis: The uncertainty levels of distributional change $V$, Vertical axis: Ratio of reduced features.}
\label{fig:exp-result-classification}
\end{figure*}

We conducted feature reduction experiments on sixteen datasets listed in Table~\ref{tab:datasets}, consisting of eight regression and binary classification tasks each.
For binary classification, we employed the logistic loss, whereas for regression we used the squared loss.
Note that we do not evaluate the prediction performance, since the proposed method safely reduces features and therefore the prediction performance is assured to be unchanged.

To model possible distributional changes, we set $\delta$ in \eqref{eq:weight_range} so that $V := \max_{\bm w\in\cW_\delta} \sum_{i\in[n]} |w_i - 1|$ is $V \in \{0, 10^{-5}, 10^{-4.5}, 10^{-4}, \dots, 10^{0}\}$. We can calculate this as $V \approx n\delta$ (see Appendix~\ref{app:delta2variety} for details).
We also varied the regularization parameter $\lambda$ over the set $\{\lambda_\mathrm{max}, 10^{-0.5}\lambda_\mathrm{max}, 10^{-1}\lambda_\mathrm{max}, 10^{-1.5}\lambda_\mathrm{max}, 10^{-2}\lambda_\mathrm{max}\}$,
where $\lambda_\mathrm{max}$ denotes the smallest value of $\lambda$ for which $\bm b^*$ being a zero vector (i.e., all features are removed). The value of $\lambda_\mathrm{max}$ depends on the dataset; its computation is detailed in Appendix~\ref{app:lambdamax}.
%

The proposed procedures for identifying inactive features consist solely of simple matrix operations and were therefore implemented using NumPy in Python.
For model training we relied on standard libraries: for binary classification, we used LIBLINEAR~ \cite{liblinear}, and for regression, we employed the  ``linear\_model.Lasso'' implementation provided by scikit-learn~\cite{scikit-learn}.
%

As the results of the experiment, we report the ratios of removed features in Figures \ref{fig:exp-result-regression} and \ref{fig:exp-result-classification}.
In each figure, the horizontal axis represents the uncertainty levels of distributional change $V$, while the vertical axis shows the ratio of features that can be removed relative to the original number of features $d$.
A value closer to one indicates that a larger fraction of features can be safely eliminated.
From the figures, we observe that a larger number of features can be removed when the uncertainty levels of distributional change is small and the regularization parameter is large.

In real uses, we set $V$ as the desired level of distribution changes, and we need to keep many features (downward in each plot) as $V$ increases.
It is worth noting that, even when $\lambda_\mathrm{max}$ is used, some features may still be identified as necessary when $V$ is increased from zero. This is attributed to the fact that the regularization parameter $\lambda_\mathrm{max}^{(\bm w)}$ defined with respect to the post-change weights $\bm w$ can differ from $\lambda_\mathrm{max}$.

%% file: sec5.tex
\section{Conclusions and Future Works}
In this paper, we studied DR feature selection under deployment-time adaptation and proposed a method called \emph{safe-DRFS}.
Our method completely eliminates the risk of discarding features under the sparse model \eqref{eq:erm} and uncertain deployment environments.

As an important direction for future work, we plan to extend the framework to forms of distributional uncertainty beyond covariate shift. Under covariate shift, distributional uncertainty is represented through instance weights; however, extending this representation to more general settings remains an open problem. Moreover, even within the covariate shift setting, it would be valuable to consider alternative weight constraints \eqref{eq:weight_range} beyond the one employed in this work.

In addition, the use of regularization functions other than L1-regularization in \eqref{eq:weighted_erm} is an important extension. Since the upper bound used for the sparseness condition (Equation \eqref{eq:upper_bound_all} and Theorem \ref{theo:main}) depends inherently on the choice of the regularization function, it is desirable to develop a more general method that accommodates a broader class of regularization functions.

Finally, when employing a machine learning method that imposes a sparseness condition on instances (such as support vector machines), it is known that training instances can be safely reduced under a similar computational principle \cite{shibagaki2016simultaneous}. Extending our approach to instance reduction in such settings is another important direction for future work.

%% file: app1.tex
\section{Proofs for the Proposed Method}

For notational simplicity, let $\bm 1$ and $\bm 0$ be vectors of all ones and all zeros, respectively.

\subsection{Lemmas to be Used}

\begin{lemma} \label{lem:strong-convexity-sphere} \cite{ndiaye2015gap}
Let $f: \mathbb{R}^k\to\mathbb{R}$ be a $\mu$-strongly-convex function, that is, $f(\bm v) - (\mu/2)\|\bm v\|_2^2$ is convex.
Let $\bm v^* := \argmin_{\bm v\in\mathbb{R}^k} f(\bm v)$. 
hen, for any $\bm v\in\mathbb{R}^k$, $\|\bm v - \bm v^*\|_2 \leq \sqrt{\frac{2}{\mu}[f(\bm v) - f(\bm v^*)]}$.
\end{lemma}

\begin{lemma}[``Rearrangement inequality''] \label{lem:rearrangement}
Let $\mathrm{sort}(\cdot)$ be the vector by sorting elements in ascending order.
For a vector $\bm s\in\mathbb{R}^n$, let $\mathcal{H}_{\bm s}$ be the set of all vectors whose elements are shuffled. Then, for any vector $\bm h\in\mathbb{R}^n$,
\begin{align*}
\max_{\bm s^\prime\in\mathcal{H}_{\bm s}} \bm h^\top \bm s^\prime = \mathrm{sort}(\bm h)^\top\mathrm{sort}(\bm s).
\end{align*}
\end{lemma}

\begin{lemma}[Corollary 32.3.4 of \cite{rockafellar1970convex}] \label{lem:maximize-in-polyhedron}
Given a convex function $f: \mathbb{R}^n\to\mathbb{R}$ and a closed convex polytope $\mathfrak{P}\subseteq\mathbb{R}^n$, the equation $\max_{\bm w\in\mathfrak{P}} f(\bm w) = \max_{\bm w\in\check{\mathfrak{P}}} f(\bm w)$ holds, where $\check{\mathfrak{P}}$ is the set of extreme points (i.e., corners) of the polytope.
\end{lemma}

\subsection{Possible Region of Dual Solution after Weight Changes using Duality Gap (Inequality \eqref{eq:duality-gap-ball})} \label{app:duality-gap-ball}

In this proof we first show the generalized case of inequality \eqref{eq:duality-gap-ball}, then \eqref{eq:duality-gap-ball} itself.

\begin{lemma}
In the primal problem \eqref{eq:weighted_erm} (Section \ref{sec:problem_setup}), suppose that $\ell(y, t)$ is differentiable with respect to $t$, and $\frac{d}{dt}\ell(y, t)$ is $\nu$-Lipschitz continuous with respect to $t$. Then, for any feasible $\hat{\bm\beta}$ and $\hat{\bm\alpha}$, the following is assured:
\begin{align*}
& \bm\alpha^{*(\bm w)} \in \left\{\bm\alpha \mid \|\bm\alpha - \hat{\bm\alpha}\|_2 \leq \sqrt{ 2\frac{\nu}{\min_{i\in[n]} w_i} G^{(\bm w)}(\hat{\bm\beta}, \hat{\bm\alpha})}\right\}.
\end{align*}
Especially, if $\bm w$ must satisfy \eqref{eq:weight_range}, the following is also assured:
\begin{align*}
& \bm\alpha^{*(\bm w)} \in \left\{\bm\alpha \mid \|\bm\alpha - \hat{\bm\alpha}\|_2 \leq \sqrt{ 2\frac{\nu}{1 - \delta} G^{(\bm w)}(\hat{\bm\beta}, \hat{\bm\alpha})}\right\}.
\tag{\eqref{eq:duality-gap-ball} restated}
\end{align*}
\end{lemma}

\begin{proof}
First we show that $G^{(\bm w)}(\hat{\bm\beta}, \hat{\bm\alpha})$ is $(\min_{i\in[n]} w_i/\nu)$-strongly convex with respect to $\hat{\bm\alpha}$.
Since $\frac{d}{dt}\ell(y, t)$ is $\nu$-Lipschitz continuous with respect to $t$, $\ell^*(y, t)$ is $(1/\nu)$-strongly convex (\cite{hiriart1993convex}, Section X.4.2).
As a result, $\cD^{(\bm w)}(\hat{\bm\alpha}) = -\sum_{i \in [n]} w_i \ell^*(y_i, -\hat{\alpha}_i)$ is $(\min_{i\in[n]} w_i/\nu)$-strongly concave with respect to $\hat{\bm\alpha}$.

Then, like the procedures in \cite{ndiaye2015gap,shibagaki2016simultaneous},
\begin{align*}
& \|\bm\alpha^{*(\bm w)} - \hat{\bm\alpha}\|_2
	\leq \sqrt{\frac{2\nu}{\min_{i\in[n]} w_i}[\cD^{(\bm w)}(\bm\alpha^{*(\bm w)}) - \cD^{(\bm w)}(\hat{\bm\alpha})]}
		&& (\because\text{Lemma \ref{lem:strong-convexity-sphere}}) \\
& = \sqrt{\frac{2\nu}{\min_{i\in[n]} w_i}[\cP^{(\bm w)}(\bm\beta^{*(\bm w)}) - \cD^{(\bm w)}(\hat{\bm\alpha})]}
	&& (\because\text{Property of dual problem}) \\
& \leq \sqrt{\frac{2\nu}{\min_{i\in[n]} w_i}[\cP^{(\bm w)}(\hat{\bm\beta}) - \cD^{(\bm w)}(\hat{\bm\alpha})]}
	= \sqrt{\frac{2\nu}{\min_{i\in[n]} w_i}G^{(\bm w)}(\hat{\bm\beta}, \hat{\bm\alpha})}.
	&& (\because \bm\beta^{*(\bm w)}~\text{is the minimizer of}~\cP^{(\bm w)})
\end{align*}
Finally, since $\min_{i\in[n]} w_i \geq 1 - \delta$ in our setup (Inequality \eqref{eq:weight_range}), we have the conclusion.
\end{proof}

\subsection{Proof of the Upper Bound of the Sparseness Condition (Theorem \ref{theo:upper_bound})} \label{app:upper_bound}

We evaluate \eqref{eq:upper_bound} in Theorem \ref{theo:upper_bound} as follows:
\begin{align*}
& \left| \sum_{i \in [n]} w_i \alpha_i^{*(\bm w)} x_{ij} \right|
	= \left| \left[ \sum_{i \in [n]} w_i \hat{\alpha}_i x_{ij} \right] + \left[ \sum_{i \in [n]} w_i x_{ij} (\alpha_i^{*(\bm w)} - \hat{\alpha}_i) \right] \right| \\
& \le  \left| \sum_{i \in [n]} w_i \hat{\alpha}_i x_{ij} \right| + \left| \sum_{i \in [n]} w_i x_{ij} (\alpha_i^{*(\bm w)} - \hat{\alpha}_i) \right| \\
& \le  \left| \sum_{i \in [n]} w_i \hat{\alpha}_i x_{ij} \right| + \sqrt{ \left( \sum_{i \in [n]} (w_i x_{ij})^2 \right) \left( \sum_{i \in [n]} (\alpha_i^{*(\bm w)} - \hat{\alpha}_i)^2 \right) }
	\quad(\because~\text{Cauchy-Schwarz inequality}) \\
& \le  \left| \sum_{i \in [n]} w_i \hat{\alpha}_i x_{ij} \right| + \sqrt{ \left( \sum_{i \in [n]} (w_i x_{ij})^2 \right) \left( \frac{2\nu}{1 - \delta}G^{(\bm w)}(\hat{\bm\beta}, \hat{\bm\alpha}) \right) }.
	\quad(\because~\text{\eqref{eq:duality-gap-ball}})
\end{align*}

\subsection{Feasibility of $\hat{\bm \alpha}$ for Weight Changes} \label{app:feasibility-alpha-hat}

As presented in \eqref{eq:beta-alpha-hat-feasible} (Section \ref{sec:safescreening}),
when we compute $\mathrm{UB}_j^{(\bm w)}(\hat{\bm \alpha}, \hat{\bm \beta})$ in Theorem \ref{theo:upper_bound},
we set $\hat{\bm \alpha} = q \bm w^{-1} \circ \bm \alpha^*$ ($q\in(0, 1]$) so that it is feasible for the dual problem after weight changes. We show the proof of the feasibility, including how to set $q$.

First, for the weights after change $\{w_i\}_{i=1}^n$, $\hat{\bm \alpha}$ must satisfy the constraints in the dual problem \eqref{eq:dual}, that is,
\begin{align*}
\left\| \sum_{i \in [n]} w_i \hat{\alpha}_i \bm x_i \right\|_\infty \le \lambda,
\quad
\sum_{i \in [n]} w_i \hat{\alpha}_i = 0.
\end{align*}
On the other hand, since $\bm \alpha^*$ is the solution before weight changes, the following constratints are assured to be satisfied:
\begin{align*}
\left\| \sum_{i \in [n]} \alpha^*_i \bm x_i \right\|_\infty \le \lambda,
\quad
\sum_{i \in [n]} \alpha^*_i = 0.
\end{align*}
So setting $\hat{\bm \alpha} = q \bm w^{-1} \circ \bm \alpha^*$ surely satisfies the constraints after weight changes if $q\in(0, 1]$.

We also need to be aware of the feasibility of the conjugate of the loss function $\ell^*(y_i, -\alpha_i)$,
since the domain of $\ell^*$ may not be any real number. So we adjust $q\in(0, 1]$ so that it is surely feasible, where $q$ is taken as large as possible so that $\hat{\bm \alpha}$ becomes similar to $\bm w^{-1} \circ \bm \alpha^*$ as possible. Here we show examples.

\begin{itemize}
\item In case of the squared loss $\ell(y, t) = (t - y)^2$, we have
	$\ell^*(y_i, -\alpha_i) = \frac{1}{4}\alpha_i^2 - y_i\alpha_i$.
	So any $\alpha_i$ is feasible, and we have only to set $q=1$.
\item  In case of the logistic loss $\ell(y, t) = \log(1 + e^{-yt})$, we have
	\begin{align*}
	&\ell^*(y_i, -\alpha_i) = \begin{cases}
		0, & (y_i\alpha_i \in \{0, 1\}) \\
		(1-y_i\alpha_i) \log(1-y_i\alpha_i) + y_i\alpha_i \log y_i\alpha_i, & (0 < y_i\alpha_i < 1) \\
		+\infty. & (\text{otherwise})
		\end{cases}
	\end{align*}
	As shown above, the feasibility condition is $0\leq y_i\alpha_i\leq 1$.
	So we need to set $q$ so that $0\leq y_i\hat{\alpha}_i\leq 1$ on condition that $0\leq y_i\alpha^*_i\leq 1$.
	To do this, since $1-\delta\leq w_i\leq 1+\delta$ is assumed, we have only to set $q = 1 - \delta$.
\end{itemize}

\subsection{Maximization of Linear and Convex Functions in the Sum and the Bound Constraints} \label{app:bounds-and-sum}

In order to prove the main theorem (Theorem \ref{theo:main}),
as defined in Section \ref{sec:problem_setup}, we consider the maximization under the following constraint:
\begin{align}
 \cW_\delta
 :=
 \left\{
 \bm w \in [1-\delta, 1+\delta]^n
 \left|
 \sum_{i \in [n]} w_i = n
 \right.
 \right\}.
 \tag{\eqref{eq:weight_range} restated}
\end{align}

For this setup, let us derive the maximization procedure.

\begin{lemma} \label{lem:corners}
The vector $\bm w$ in constraints \eqref{eq:weight_range} composes a convex polytope on a hyperplane.
Their corners are represented as follows:
\begin{itemize}
\item If $n$ is an even number, $\bm w$ is at a corner if and only if $\mathrm{sort}(\bm w) = [\underbrace{1-\delta, \dots, 1-\delta}_{(n/2)~\text{elements}}, \underbrace{1+\delta, \dots, 1+\delta}_{(n/2)~\text{elements}}]$. (There exist ${}_n C_{n/2}$ corners.)
\item If $n$ is an odd number, $\bm w$ is at a corner if and only if $\mathrm{sort}(\bm w) = [\underbrace{1-\delta, \dots, 1-\delta}_{[(n-1)/2]~\text{elements}}, 1, \underbrace{1+\delta, \dots, 1+\delta}_{[(n-1)/2]~\text{elements}}]$. (There exist ${}_n C_{(n-1)/2}\cdot \frac{n+1}{2}$ corners.)
\end{itemize}
\end{lemma}

\begin{proof}
First, it is obvious that $\bm w$ is on a hyperplane thanks to the constraint $\bm w^\top \bm 1 = n$.
Also, since the constraint $\forall i\in[n]:~\delta \geq w_i - 1 \geq -\delta$ composes a hypercube, we have only to identify all intersections between the hyperplane $\bm w^\top \bm 1 = n$ and ``all edges in the hypercube''.

The edges of the hypercube are represented as follows:
\begin{align*}
& \{ k\in[n], \varepsilon\in\{-1, +1\}^n \mid
	1 - \delta\leq w_k \leq 1 + \delta,
	\quad
	\forall k^\prime\in[n]\setminus\{k\}:~w_{k^\prime} = 1 + \varepsilon_{k^\prime}\delta \}.
\end{align*}
For each such edge, it has an intersection with the hyperplane $\bm w^\top \bm 1 = n$ under the following conditions:
\begin{itemize}
\item If $n$ is an even number, among $\{\varepsilon_{k^\prime}\}_{k^\prime\in[n]\setminus\{k\}}$ ($n-1$ elements), the number of $+1$ and $-1$ must differ by one. If the number of $+1$ is larger (resp. smaller) by one than that of $-1$, $w_k$ is set to $1-\delta$ (resp. $1+\delta$). \\
	In summary, among $\{w_i\}_{i\in[n]}$, half of them must be set as $1+\delta$ while others as $1-\delta$ to represent a corner.
\item If $n$ is an odd number, among $\{\varepsilon_{k^\prime}\}_{k^\prime\in[n]\setminus\{k\}}$ ($n-1$ elements), the number of $+1$ and $-1$ must be the same. $w_k$ is set to $1$. \\
	In summary, among $\{w_i\}_{i\in[n]}$, half (rounded down) of them must be set as $1+\delta$, one as $1$, and others as $1-\delta$ to represent a corner.
\end{itemize}
\end{proof}

Combining these facts, we have the following fact:

\begin{lemma} \label{lem:bounds-and-sum}
Let $\bm c\in\RR^n$, and $\cW_\delta$ be the one in \eqref{eq:weight_range}.
Let $c_{(i)}$ be the $i^\mathrm{th}$-smallest element in $\bm c$, and
\begin{align*}
	w_{(i)} = \begin{cases}
		1 - \delta, & (i \leq n/2) \\
		1, & (i = (n-1)/2) \\
		1 + \delta. & (i \geq n/2)
		\end{cases}
\end{align*}
Then,
\begin{itemize}
\item $\max_{\bm w\in\cW_\delta} \bm c^\top\bm w = \sum_{i\in[n]} c_{(i)} w_{(i)}$.
\item $\max_{\bm w\in\cW_\delta} \bm c^\top(\bm w \circ \bm w) = \sum_{i\in[n]} c_{(i)} w_{(i)}^2$.
\end{itemize}
\end{lemma}

\begin{proof}
First we can easily confirm that the both expressions in $\max_{\bm w\in\cW_\delta}$ is both convex with respect to $\bm w$.
So, thanks to Lemma \ref{lem:maximize-in-polyhedron}, we have only to evaluate the expressions at all corners identified by Lemma \ref{lem:corners}.
Then, both expressions can be maximized by Lemma \ref{lem:rearrangement}, where the procedure is as defined in the description of the lemma.
\end{proof}

\subsection{Proof of Main Theorem (Theorem \ref{theo:main})} \label{app:proof-main}

First, in order to compute \eqref{eq:upper_bound_all},
we substitute $\hat{\bm \beta}$, $\hat{\bm \alpha}$ with the one specified in \eqref{eq:beta-alpha-hat-feasible} (Section \ref{sec:safescreening}).
Then we have
\begin{align*}
& \mathrm{UB}_j^{(\bm w)}(\bm \beta^*, q \bm w^{-1} \circ \bm \alpha^*)
 = 
  \left| \sum_{i \in [n]} q \alpha^*_i x_{ij} \right|
  + \sqrt{
	\left( \sum_{i \in [n]} w_i^2 x_{ij}^2 \right)
	\left( \frac{2 \nu}{1-\delta} G^{(\bm w)}\left(\bm \beta^*, q \bm w^{-1} \circ \bm \alpha^* \right) \right)
  } \\
 & =
  q \left| \sum_{i \in [n]} \alpha^*_i x_{ij} \right|
  + \sqrt{
	\left( \sum_{i \in [n]} w_i^2 x_{ij}^2 \right)
	\left( \frac{2 \nu}{1-\delta} \left( \sum_{i \in [n]} w_i \left( \ell\bigl(y_i, \bm x_i^\top \bm b + b_0 \bigr) + \ell^*\left(y_i, -\frac{q}{w_i}\alpha_i\right) \right)+ \lambda \|\bm b\|_1 \right) \right)
  } \\
 & \le
  q \left| \sum_{i \in [n]} \alpha^*_i x_{ij} \right| \\
  &\quad + \sqrt{
	\left( \sum_{i \in [n]} w_i^2 x_{ij}^2 \right)
	\left( \frac{2 \nu}{1-\delta} \left( \sum_{i \in [n]} w_i \left( \ell\bigl(y_i, \bm x_i^\top \bm b + b_0 \bigr) +
		\max\left\{ \ell^*\left(y_i, -\frac{q}{1+\delta}\alpha_i\right), \ell^*\left(y_i, -\frac{q}{1-\delta}\alpha_i\right)  \right\}
		\right)+ \lambda \|\bm b\|_1 \right) \right)
  } \\
  & (\because \ell^*~\text{is convex with respect to second argument; the maximum is attained at either edge}) \\
 & =
  q \left| \sum_{i \in [n]} \alpha^*_i x_{ij} \right|
  + \sqrt{
	\left( \sum_{i \in [n]} \theta^{\langle j\rangle}_i w_i^2 \right)
	\left( \frac{2 \nu}{1-\delta} \left( \sum_{i \in [n]} \rho_i w_i + \lambda \|\bm b\|_1 \right) \right)
  }.
\end{align*}

Finally, combining Lemma \ref{lem:bounds-and-sum} we obtain the conclusion of Theorem \ref{theo:main}.

\subsection{The Full Procedure of the Proposed Method} \label{app:full-procedure}

\begin{itemize}
\item Set $q=1$ (for squared loss) or $q=1-\delta$ (for logistic loss), as discussed in Appendix \ref{app:feasibility-alpha-hat}.
\item Let $N_\mathrm{max}(j)$ be the maximized value of $\| \bm w \circ \bm x_{:j} \|_2$ in the constraint \eqref{eq:weight_range}.
	First, since $\| \bm w \circ \bm x_{:j} \|_2$ is convex with respect to $\bm w$, thanks to Lemma \ref{lem:maximize-in-polyhedron} we have only to maximize it by examining the corners in Lemma \ref{lem:corners}.
	In addition, since
	\begin{align*}
	\| \bm w \circ \bm x_{:j} \|_2 = \sqrt{ (\bm w\circ\bm w)^\top (\bm x_{:j} \circ \bm x_{:j}) },
	\end{align*}
	it can be maximized by Lemma \ref{lem:rearrangement} by setting
	\begin{align*}
	& \bm s = \begin{cases}
		[\underbrace{(1-\delta)^2, \dots, (1-\delta)^2}_{(n/2)~\text{elements}}, \underbrace{(1+\delta)^2, \dots, (1+\delta)^2}_{(n/2)~\text{elements}}], & (n:~\text{even}) \\
		[\underbrace{(1-\delta)^2, \dots, (1-\delta)^2}_{[(n-1)/2]~\text{elements}}, 1, \underbrace{(1+\delta)^2, \dots, (1+\delta)^2}_{[(n-1)/2]~\text{elements}}], & (n:~\text{odd})
		\end{cases}
		\\
	& \bm h = \bm x_{:j} \circ \bm x_{:j}.
	\end{align*}
\item Compute $\bar{G}^{(\bm w)}$ as
	\begin{align*}
	& \bar{G}^{(\bm w)} = \bm w^\top \bm l^{*(\bm 1)} + \lambda\|\bm b^{*(\bm 1)}\|_1,
	\end{align*}
	where $\bm l^{*(\bm 1)} \in\mathbb{R}^n$ is defined as
	\begin{align*}
	l^{*(\bm 1)}_i = & \ell\bigl(y_i, \bm x_i^\top \bm b^{*(\bm 1)} + b_0^{*(\bm 1)} \bigr) + \max\Bigl\{ \ell^*\left(y_i, -\frac{q}{1 + \delta}\alpha^{*(\bm 1)}_i\right), \ell^*\left(y_i, -\frac{q}{1 - \delta}\alpha^{*(\bm 1)}_i\right) \Bigr\}.
	\end{align*}
\item let $\bar{G}_\mathrm{max}$ be the maximized value of $\bar{G}^{(\bm w)}$ in the constraint \eqref{eq:weight_range}.
	First, since $\bar{G}^{(\bm w)}$ is convex (and also linear) with respect to $\bm w$, thanks to Lemma \ref{lem:maximize-in-polyhedron} we have only to maximize it by examining the corners in Lemma \ref{lem:corners}.
	It can be maximized by Lemma \ref{lem:rearrangement} by setting
	\begin{align*}
	& \bm s = \begin{cases}
		[\underbrace{1-\delta, \dots, 1-\delta}_{(n/2)~\text{elements}}, \underbrace{1+\delta, \dots, 1+\delta}_{(n/2)~\text{elements}}], & (\text{$n$ is an even number}) \\
		[\underbrace{1-\delta, \dots, 1-\delta}_{[(n-1)/2]~\text{elements}}, 1, \underbrace{1+\delta, \dots, 1+\delta}_{[(n-1)/2]~\text{elements}}], & (\text{$n$ is an odd number})
		\end{cases}
		\\
	& \bm h = \bm l^{*(\bm 1)}.
	\end{align*}
\item If $q \left| \sum_{i \in [n]} \alpha^{*(\bm 1)}_i x_{ij} \right| + N_\mathrm{max}(j) \sqrt{\frac{2 \nu}{1 - \delta} \bar{G}_\mathrm{max}}  < \lambda$, we can safely remove the $j^\mathrm{th}$ feature.
\end{itemize}

%% file: app2.tex
\section{Detailed Setups of the Experiment} \label{app:experiment}

\subsection{Preparation of Datasets} \label{app:experiment-datasets}

The processes are common in all datasets:
\begin{itemize}
\item We removed features beforehand (i.e., not counted in the number of features $d$) if its values are unique.
\item For each feature in a dataset, their values are normalized to average 0 and (sample) standard deviation 1 by linear transformation. Note that the outcome variable is not normalized.
\end{itemize}

For LIBSVM datasets, we used only training set if a dataset is separated into training and testing datasets.
Also, we used ``scale'' datasets if both scaled and un-scaled datasets are provided.

For datasets from UCI Machine Learning Repository, the following preprocesses are applied:
\begin{itemize}
\item We removed all non-predictive features (e.g., just representing the name).
\item We removed all features having missing values.
\item For ``superconductivity'' dataset, since there are two files of features ("train" and "unique\_m"), we just merged them.
\end{itemize}

Due to space limitations, we abbreviated names of UCI datasets.
For avoidance of doubt, we show the sources as follows:
\begin{itemize}
\item facebook\_comment: \url{https://archive.ics.uci.edu/dataset/363/facebook+comment+volume+dataset}
\item online\_news: \url{https://archive.ics.uci.edu/dataset/332/online+news+popularity}
\item communities\_and\_crime: \url{https://archive.ics.uci.edu/dataset/183/communities+and+crime}
\item superconductivity: \url{https://archive.ics.uci.edu/dataset/464/superconductivty+data}
\item blog\_feedback: \url{https://archive.ics.uci.edu/dataset/304/blogfeedback}
\end{itemize}

\subsection{Conversions between the Bounds of Each Instance Weight $\delta$ and the Total Instance Weight Changes $V$} \label{app:delta2variety}

As shown in \eqref{eq:weight_range} (Section \ref{sec:problem_setup}),
we set $\delta\in(0, 1)$ to limit the possible weight changes in \eqref{eq:weight_range}.
Under the constraint, $V = \max_{\bm w\in\cW_\delta} \|\bm w - \bm 1\|_1$ can be computed as follows.

Since $\|\bm w - \bm 1\|_1$ is convex with respect to $\bm w$,
thanks to Lemmas \ref{lem:maximize-in-polyhedron} and \ref{lem:corners},
it is maximized by one of $\bm w$ such that
\begin{align*}
\mathrm{sort}(\bm w) = \begin{cases}
	[\underbrace{1-\delta, \dots, 1-\delta}_{(n/2)~\text{elements}}, \underbrace{1+\delta, \dots, 1+\delta}_{(n/2)~\text{elements}}]
		& (\text{$n$ is an even number}), \\
	[\underbrace{1-\delta, \dots, 1-\delta}_{[(n-1)/2]~\text{elements}}, 1, \underbrace{1+\delta, \dots, 1+\delta}_{[(n-1)/2]~\text{elements}}]
		& (\text{$n$ is an odd number}).
	\end{cases}
\end{align*}
Since $\|\bm w - \bm 1\|_1$ is invariant to the sorting of $\bm w$, we can easily compute that
\begin{align*}
V = \begin{cases}
		n\delta & (\text{$n$ is an even number}), \\
		(n-1)\delta & (\text{$n$ is an odd number}).
	\end{cases}
\end{align*}

\subsection{$\lambda_\mathrm{max}$ for L1-regularization} \label{app:lambdamax}

If we employ L1-regularized learning \eqref{eq:weighted_erm} (Section \ref{sec:problem_setup}), we can find \emph{the smallest $\lambda$ that attains $\bm b^{*(\bm w)} = \bm 0$}, that is, all features are made inactive. Such $\lambda$ is often referred to as $\lambda_\mathrm{max}$. $\lambda_\mathrm{max}$ is often used to determine the baseline of $\lambda$'s in a data-dependent manner \cite{friedman2010regularization}.

Given a dataset $\{(\bm x_i, y_i)\}_{i=1}^n$ and the weights $\bm w$ as in \eqref{eq:weighted_erm}, we show how to compute $\lambda_\mathrm{max}$, basing on the manner in \cite{ndiaye2017gap} but following the notations employed in the paper.
To compute $\lambda_\mathrm{max}$, we focus on the expression \eqref{eq:sparseness_condition}. To achieve $\bm b^{*(\bm w)} = \bm 0$, the precondition of \eqref{eq:sparseness_condition} must hold for all $j\in[d]$, that is,
\begin{align*}
& \forall j\in[d]:\quad \left|\sum_{i\in[n]} x_{ij} w_i \alpha^{*(\bm w)}_i\right| < \lambda_\mathrm{max}, \\
& \therefore \lambda_\mathrm{max} = \max_{j\in[d]} \left|\sum_{i\in[n]} x_{ij} w_i \alpha^{*(\bm w)}_i\right|.
\end{align*}
Here, we need $\bm\alpha^{*(\bm w)}$ to compute this, which can be easily computed as follows:
\begin{enumerate}
\item First we solve the problem \eqref{eq:weighted_erm} on condition that $\bm b^{*(\bm w)} = \bm 0$, that is, 
	$b^{*(\bm w)}_0 := \argmin_{b_0 \in\mathbb{R}} \sum_{i=1}^n w_i \ell(y_i, b_0)$.
\item In \eqref{eq:weighted_erm}, it is known that
	\begin{align*}
	-\alpha^{*(\bm w)}_i \in \partial\ell(y_i, \bm x_i^\top \bm b^{*(\bm w)} + b^{*(\bm w)}_0).
	\end{align*}
	(See \cite{ndiaye2015gap} for example).
	So we can compute $\bm\alpha^{*(\bm w)}$ on condition that $\bm b^{*(\bm w)} = \bm 0$ as
	$-\alpha^{*(\bm w)}_i \in \partial\ell(y_i, b^{*(\bm w)}_0)$.
\end{enumerate}

\begin{corollary}[$\lambda_\mathrm{max}$ for the squared loss]
Let us compute $\lambda_\mathrm{max}$ for the squared loss $\ell(y_i, t) := (t - y_i)^2$ (which we employed in the experiments) for binary classifications. Since
\begin{align*}
\sum_{i=1}^n w_i \ell(y_i, b_0) = \sum_{i=1}^n w_i (b_0 - y_i)^2,
\end{align*}
its minimizer can be analytically obtained as
\begin{align*}
b^{*(\bm w)}_0 = \frac{\sum_{i=1}^n w_i y_i}{\sum_{i=1}^n w_i}.
\end{align*}
Then we have
\begin{align*}
&- \alpha^{*(\bm w)}_i \in \partial\ell(y_i, b^{*(\bm w)}_0)
	= 2(b^{*(\bm w)}_0 - y_i), \\
&\therefore\quad
	\alpha^{*(\bm w)}_i = 2(y_i - b^{*(\bm w)}_0).
\end{align*}
Finally we have
\begin{align*}
& \lambda_\mathrm{max}
	= \max_{j\in[d]} \left|\sum_{i\in[n]} x_{ij} w_i \alpha^{*(\bm w)}_i\right|
	= 2 \max_{j\in[d]} \left|\sum_{i\in[n]} x_{ij} w_i (y_i - b^{*(\bm w)}_0)\right|.
\end{align*}
\end{corollary}

\begin{corollary}[$\lambda_\mathrm{max}$ for the logistic loss]
Let us compute $\lambda_\mathrm{max}$ for the logistic loss $\ell(y_i, t) := \log(1 + e^{-yt})$ (which we employed in the experiments) for binary classifications ($y_i \in\{-1, +1\}$). Since
\begin{align*}
\sum_{i=1}^n w_i \ell(y_i, b_0) = \sum_{i=1}^n w_i \log(1 + e^{- y_i b_0}),
\end{align*}
its minimizer can be analytically obtained as
\begin{align*}
b^{*(\bm w)}_0 = \log\biggl(\sum_{i:~y_i = +1} w_i\biggr) - \log\biggl(\sum_{i:~y_i = -1} w_i\biggr).
\end{align*}
Then we have
\begin{align*}
&-y_i \alpha^{*(\bm w)}_i \in \partial\ell(y_i, b^{*(\bm w)}_0)
	= \frac{-e^{-y_i b^{*(\bm w)}_0}}{1 + e^{-y_i b^{*(\bm w)}_0}}
	= -\frac{1}{e^{y_i b^{*(\bm w)}_0} + 1} \\
& \therefore\quad
	\alpha^{*(\bm w)}_i = \frac{y_i}{e^{y_i b^{*(\bm w)}_0} + 1}.
\end{align*}
Finally we have
\begin{align*}
& \lambda_\mathrm{max}
	= \max_{j\in[d]} \left|\sum_{i\in[n]} w_i x_{ij} \alpha^{*(\bm w)}_i\right|
	= \max_{j\in[d]} \left|\sum_{i\in[n]} \frac{w_i x_{ij} y_i}{e^{y_i b^{*(\bm w)}_0} + 1}\right|.
\end{align*}
\end{corollary}

%% file: paper.bib
@inproceedings{namkoong2016stochastic,
  title={Stochastic gradient methods for distributionally robust optimization with f-divergences},
  author={Namkoong, Hongseok and Duchi, John},
  booktitle={Advances in Neural Information Processing Systems},
  year={2016}
}

@inproceedings{shafieezadeh2015distributionally,
  title={Distributionally Robust Logistic Regression},
  author={Shafieezadeh-Abadeh, Soroosh and Mohajerin Esfahani, Peyman and Kuhn, Daniel},
  booktitle={Advances in Neural Information Processing Systems},
  year={2015}
}

@article{mohajerin2018data,
  title={Data-driven distributionally robust optimization using the Wasserstein metric: Performance guarantees and tractable reformulations},
  author={Mohajerin Esfahani, Peyman and Kuhn, Daniel},
  journal={Mathematical Programming},
  volume={171},
  number={1},
  pages={115--166},
  year={2018},
  publisher={Springer}
}

@inproceedings{namkoong2017variance,
  title={Variance-Based Regularization with Convex Objectives},
  author={Namkoong, Hongseok and Duchi, John C},
  booktitle={Advances in Neural Information Processing Systems},
  year={2017}
}

@inproceedings{sinha2018certifying,
  title={Certifying Some Distributional Robustness with Principled Adversarial Training},
  author={Sinha, Aman and Namkoong, Hongseok and Duchi, John},
  booktitle={International Conference on Learning Representations},
  year={2018}
}

@inproceedings{wang2021tent,
  title={Tent: Fully Test-Time Adaptation by Entropy Minimization},
  author={Wang, Dequan and Shelhamer, Evan and Liu, Shaoteng and Olshausen, Bruno and Darrell, Trevor},
  booktitle={International Conference on Learning Representations},
  year={2021}
}

@inproceedings{sun2020test,
  title={Test-time training with self-supervision for generalization under distribution shifts},
  author={Sun, Yu and Wang, Xiaolong and Liu, Zhuang and Miller, John and Efros, Alexei and Hardt, Moritz},
  booktitle={International conference on machine learning},
  pages={9229--9248},
  year={2020},
  organization={PMLR}
}

@article{zhou2021domain,
  title={Domain adaptive ensemble learning},
  author={Zhou, Kaiyang and Yang, Yongxin and Qiao, Yu and Xiang, Tao},
  journal={IEEE Transactions on Image Processing},
  volume={30},
  pages={8008--8018},
  year={2021},
  publisher={IEEE}
}

@inproceedings{swaroop2025distributionally,
  title={Distributionally Robust Feature Selection},
  author={Swaroop, Maitreyi and Krishnamurti, Tamar and Wilder, Bryan},
  booktitle={Advances in Neural Information Processing Systems},
  year={2025}
}

@inproceedings{liu2014safe,
  title     = {Safe Screening with Variational Inequalities and Its Application to Lasso},
  author    = {Liu, Jun and Zhao, Zheng and Wang, Jie and Ye, Jieping},
  booktitle = {Proceedings of the 31st International Conference on Machine Learning (ICML)},
  year      = {2014},
  pages     = {179--187},
  publisher = {PMLR}
}

@inproceedings{wang2014safe,
  title     = {A Safe Screening Rule for Sparse Logistic Regression},
  author    = {Wang, Jie and Zhou, Jiayu and Liu, Jun and Wonka, Peter and Ye, Jieping},
  booktitle = {Advances in Neural Information Processing Systems (NeurIPS)},
  year      = {2014},
  volume    = {27}
}

@article{xiang2011screening,
  title   = {Screening Tests for Lasso Problems},
  author  = {Xiang, Zhen James and Wang, Yun and Ramadge, Peter J.},
  journal = {IEEE Transactions on Pattern Analysis and Machine Intelligence},
  year    = {2011},
  volume  = {35},
  number  = {8},
  pages   = {1905--1918}
}

@article{elghaoui2012safe,
  title={Safe Feature Elimination for the Lasso and Sparse Supervised Learning Problems},
  author={El Ghaoui, Laurent and Viallon, Vivien and Rabbani, Tarek},
  journal={Pacific Journal of Optimization},
  volume={8},
  number={4},
  pages={667--698},
  year={2012}
}

@article{wang2013lasso,
  title={Lasso screening rules via dual polytope projection},
  author={Wang, Jie and Zhou, Jiayu and Wonka, Peter and Ye, Jieping},
  journal={Advances in neural information processing systems},
  volume={26},
  year={2013}
}

@article{fercoq2015mind,
  title={Mind the Duality Gap: Safer Rules for the Lasso},
  author={Fercoq, Olivier and Gramfort, Alexandre and Salmon, Joseph},
  journal={International Conference on Machine Learning},
  year={2015}
}

@inproceedings{ndiaye2015gap,
  title={GAP Safe screening rules for sparse multi-task and multi-class models},
  author={Ndiaye, Eugene and Fercoq, Olivier and Gramfort, Alexandre and Salmon, Joseph},
  booktitle={Advances in Neural Information Processing Systems},
  pages={811--819},
  year={2015}
}

@article{ndiaye2017gap,
  title={Gap Safe Screening Rules for Sparsity Enforcing Penalties},
  author={Ndiaye, Eugene and Fercoq, Olivier and Gramfort, Alexandre and Salmon, Joseph},
  journal={Journal of Machine Learning Research},
  volume={18},
  number={128},
  pages={1--57},
  year={2017}
}

@inproceedings{ogawa2013safe,
  title={Safe Screening of Non-Support Vectors in Pathwise {SVM} Computation},
  author={Ogawa, Kohei and Suzuki, Yoshiki and Takeuchi, Ichiro},
  booktitle={Proceedings of the 30th International Conference on Machine Learning},
  year={2013}
}

@inproceedings{nakagawa2016safe,
  title={Safe Pattern Pruning: An Efficient Approach for Predictive Pattern Mining},
  author={Nakagawa, Kazuya and Suzumura, Shinya and Karasuyama, Masayuki and Tsuda, Koji and Takeuchi, Ichiro},
  booktitle={Proceedings of the 22nd ACM SIGKDD International Conference on Knowledge Discovery and Data Mining},
  year={2016}
}

@article{kato2023saferulefit,
  title={Safe RuleFit: Learning Optimal Sparse Rule Model by Meta Safe Screening},
  author={Kato, Hiroshi and Hanada, Hiroyuki and Takeuchi, Ichiro},
  journal={IEEE Transactions on Pattern Analysis and Machine Intelligence},
  year={2023}
}

@inproceedings{shibagaki2016simultaneous,
  title={Simultaneous Safe Screening of Features and Samples in Doubly Sparse Modeling},
  author={Shibagaki, Atsushi and Karasuyama, Masayuki and Hatano, Kohei and Takeuchi, Ichiro},
  booktitle={Proceedings of the 33rd International Conference on Machine Learning},
  year={2016}
}

@inproceedings{shibagaki2015regularization,
  title={Regularization Path of Cross-Validation Error Lower Bounds},
  author={Shibagaki, Atsushi and Suzuki, Yoshiki and Karasuyama, Masayuki and Takeuchi, Ichiro},
  booktitle={Advances in Neural Information Processing Systems},
  year={2015}
}

@inproceedings{okumura2015quick,
  title={Quick Sensitivity Analysis for Incremental Data Modification and Its Application to Leave-One-Out Cross-Validation in Linear Classification Problems},
  author={Okumura, Shota and Suzuki, Yoshiki and Takeuchi, Ichiro},
  booktitle={Proceedings of the 21st ACM SIGKDD International Conference on Knowledge Discovery and Data Mining},
  year={2015}
}

@inproceedings{hanada2018efficiently,
  title={Efficiently Evaluating Small Data Modification Effect for Large-Scale Classification in Changing Environment},
  author={Hanada, Hiroyuki and Shibagaki, Atsushi and Sakuma, Jun and Takeuchi, Ichiro},
  booktitle={Proceedings of the 32nd AAAI Conference on Artificial Intelligence},
  year={2018}
}

@inproceedings{ndiaye2019approximatehomotopy,
  title={Computing Full Conformal Prediction Sets with Approximate Homotopy},
  author={Ndiaye, Eugene and Takeuchi, Ichiro},
  booktitle={Advances in Neural Information Processing Systems},
  year={2019}
}

@article{shimodaira2000improving,
  title={Improving Predictive Inference under Covariate Shift by Weighting the Log-Likelihood Function},
  author={Shimodaira, Hidetoshi},
  journal={Journal of Statistical Planning and Inference},
  volume={90},
  number={2},
  pages={227--244},
  year={2000}
}

@article{sugiyama2007covariate,
  title={Covariate Shift Adaptation by Importance Weighted Cross Validation},
  author={Sugiyama, Masashi and Krauledat, Matthias and M{\"u}ller, Klaus-Robert},
  journal={Journal of Machine Learning Research},
  volume={8},
  pages={985--1005},
  year={2007}
}

@article{kouw2019review,
  title={A Review of Domain Adaptation without Target Labels},
  author={Kouw, Wouter M. and Loog, Marco},
  journal={IEEE Transactions on Pattern Analysis and Machine Intelligence},
  volume={43},
  number={3},
  pages={766--785},
  year={2019}
}

@article{scikit-learn,
  title={Scikit-learn: Machine Learning in {P}ython},
  author={Pedregosa, F. and Varoquaux, G. and Gramfort, A. and Michel, V.
          and Thirion, B. and Grisel, O. and Blondel, M. and Prettenhofer, P.
          and Weiss, R. and Dubourg, V. and Vanderplas, J. and Passos, A. and
          Cournapeau, D. and Brucher, M. and Perrot, M. and Duchesnay, E.},
  journal={Journal of Machine Learning Research},
  volume={12},
  pages={2825--2830},
  year={2011}
}

@article{liblinear,
 author = {Fan, Rong-En and Chang, Kai-Wei and Hsieh, Cho-Jui and Wang, Xiang-Rui and Lin, Chih-Jen},
 journal = {Journal of Machine Learning Research},
 pages = {1871--1874},
 title = {{LIBLINEAR}: A library for large linear classification},
 volume = {9},
 year = {2008}
 }

@article{libsvmDataset,
  title={LIBSVM: A library for support vector machines},
  author={Chang, Chih-Chung and Lin, Chih-Jen},
  journal={ACM Transactions on Intelligent Systems and Technology (TIST)},
  volume={2},
  number={3},
  pages={27},
  year={2011},
  publisher={ACM},
  note={Datasets are provided in authors' website: \url{https://www.csie.ntu.edu.tw/~cjlin/libsvmtools/datasets/}.}
}

@misc{Dua2017UCI,
author = "Dheeru, Dua and Karra Taniskidou, Efi",
year = "2017",
title = "{UCI} Machine Learning Repository",
url = "http://archive.ics.uci.edu/",
institution = "University of California, Irvine, School of Information and Computer Sciences"
}

@book{hiriart1993convex,
  title={Convex Analysis and Minimization Algorithms II: Advanced Theory and Bundle Methods},
  author={Hiriart-Urruty, Jean-Baptiste and Lemar{\'e}chal, Claude},
  publisher={Springer},
  year={1993}
}

@book{rockafellar1970convex,
  title={Convex analysis},
  author={Rockafellar, Ralph Tyrell},
  year={1970},
  publisher={Princeton university press},
}

@article{friedman2010regularization,
 author = {J. Friedman and T. Hastie and R. Tibshirani},
 journal = {Journal of Statistical Software},
 number = {1},
 pages = {1--22},
 title = {Regularization Paths for Generalized Linear Models via Coordinate Descent},
 volume = {33},
 year = {2010},
 }
